%% file: robust.tex
\documentclass[conference]{IEEEtran}

\usepackage[utf8]{inputenc}
\usepackage{amsmath,amssymb}

\usepackage{makecell}
\usepackage{multirow}
\usepackage{algorithm}
\usepackage{algpseudocode}
\usepackage{todonotes}
\usepackage[export]{adjustbox}
\usepackage{xcolor}

\usepackage[miktex]{gnuplottex}

\usepackage{epstopdf}

\algtext*{EndIf}%
\algtext*{EndFor}%
\algnewcommand{\LineComment}[1]{\State \(\triangleright\) #1}

\usepackage{amsthm}

\newtheorem{proposition}{Proposition}

\DeclareMathOperator*{\argmin}{arg\,min}

\usepackage{microtype}
\usepackage{graphicx}
\usepackage{subfigure}
\usepackage{booktabs} %

\usepackage{hyperref}

\def\BibTeX{{\rm B\kern-.05em{\sc i\kern-.025em b}\kern-.08em
    T\kern-.1667em\lower.7ex\hbox{E}\kern-.125emX}}

\newcommand{\txt}[1]{#1}    

\begin{document}

\title{Robust Training of Vector Quantized\\Bottleneck Models
}

\makeatletter
\newcommand{\AND}{%
  \end{@IEEEauthorhalign}
  \hfill\mbox{}\par
  \mbox{}\hfill\begin{@IEEEauthorhalign}
}
\makeatother

\author{%
\IEEEauthorblockN{\makebox[0.33\textwidth][c]{Adrian Łańcucki}}
\IEEEauthorblockA{\makebox[0.33\textwidth][c]{\textit{NVIDIA Corporation}}}
\IEEEauthorblockA{\makebox[0.33\textwidth][c]{Warsaw, Poland}}
\and
\IEEEauthorblockN{\makebox[0.33\textwidth][c]{Jan Chorowski}}
\IEEEauthorblockA{\makebox[0.33\textwidth][c]{\textit{University of Wrocław}}}
\IEEEauthorblockA{\makebox[0.33\textwidth][c]{Wrocław, Poland}}
\and
\IEEEauthorblockN{\makebox[0.33\textwidth][c]{Guillaume Sanchez}}
\IEEEauthorblockA{\makebox[0.33\textwidth][c]{\textit{Université de Toulon, LIS}}}
\IEEEauthorblockA{\makebox[0.33\textwidth][c]{Toulon, France}}
\AND
\IEEEauthorblockN{\makebox[0.33\textwidth][c]{Ricard Marxer}}
\IEEEauthorblockA{\makebox[0.33\textwidth][c]{\textit{Université de Toulon, LIS}}}
\IEEEauthorblockA{\makebox[0.33\textwidth][c]{Toulon, France}}
\and
\IEEEauthorblockN{\makebox[0.33\textwidth][c]{Nanxin Chen}}
\IEEEauthorblockA{\makebox[0.33\textwidth][c]{\textit{Johns Hopkins University}}}
\IEEEauthorblockA{\makebox[0.33\textwidth][c]{Baltimore, USA}}
\and
\IEEEauthorblockN{\makebox[0.33\textwidth][c]{Hans J.G.A. Dolfing}}
\AND
\IEEEauthorblockN{\makebox[0.33\textwidth][c]{Sameer Khurana}}
\IEEEauthorblockA{\makebox[0.33\textwidth][c]{\textit{Massachusetts Institute of Technology}}}
\IEEEauthorblockA{\makebox[0.33\textwidth][c]{Cambridge, USA}}
\and
\IEEEauthorblockN{\makebox[0.33\textwidth][c]{Tanel Alumäe}}
\IEEEauthorblockA{\makebox[0.33\textwidth][c]{\textit{Tallinn University of Technology}}}
\IEEEauthorblockA{\makebox[0.33\textwidth][c]{Tallinn, Estonia}}
\and
\IEEEauthorblockN{\makebox[0.33\textwidth][c]{Antoine Laurent}}
\IEEEauthorblockA{\makebox[0.33\textwidth][c]{\textit{Le Mans University}}}
\IEEEauthorblockA{\makebox[0.33\textwidth][c]{Le Mans, France}}
}

\maketitle

\begin{abstract}
In this paper we demonstrate methods for reliable and efficient training of discrete representation using Vector-Quantized Variational Auto-Encoder models (VQ-VAEs).
Discrete latent variable models have been shown to learn \txt{nontrivial} representations of speech, \txt{applicable to} unsupervised voice conversion and reaching state-of-the-art performance on unit discovery tasks. For unsupervised representation learning,
they became viable alternatives to continuous latent variable models such as the Variational Auto-Encoder (VAE).
However, training deep discrete variable models is challenging, due to the inherent non-differentiability of the discretization operation.
In this paper we focus on VQ-VAE, a state-of-the-art discrete bottleneck model shown to perform on par with its continuous counterparts.
It quantizes encoder outputs with on-line $k$-means clustering.
We show that the codebook learning can suffer from poor initialization
and non-stationarity of clustered encoder outputs.
\txt{We demonstrate that these can be successfully overcome by increasing the learning rate for the codebook and periodic date-dependent codeword re-initialization.}
As a result, we achieve more robust training across different \txt{tasks}, %
and significantly increase the usage of latent codewords even for large codebooks.
This has practical benefit, for instance, in
unsupervised representation learning,
where large codebooks may lead to disentanglement
of latent representations.
\end{abstract}


\section{Introduction}
Automatic extraction of useful features is one of the hallmarks of intelligent systems. Well designed models generalize better by filtering their inputs to retain only the information that is relevant to the task at hand~\cite{tishby1999information}. Bottleneck layers, which induce this behavior, are especially important in %
feature learning.
Simple bottlenecks indirectly reduce the amount of information allowed through by limiting feature dimensionality~\cite{vesely2012language,yu2011improved}. The amount of information can be constrained directly, as in variational autoencoders,
where a continuous stochastic bottleneck represents each sample with a probability distribution~\cite{kingma_auto_2013}.

Discrete latent variable models also offer direct control over the information content of the learned representation. Recently, such models have shown great results in unsupervised learning of speech and image representations. A very popular model, the Vector-Quantized Variational Auto-Encoder (VQ-VAE) \cite{oord_neural_2017} encodes a speech signal as a sequence of discrete latent units. Compared with continuous latent variable models, the discrete VQ-VAE units have been empirically shown to correlate well with the phonetic content of the utterance, allowing advanced speech manipulation such as voice conversion~\cite{oord_neural_2017}. Discrete representations learned by VQ-VAE were shown to \txt{achieve} state-of-the-art results on unit discovery tasks, yielding a good separation between speaker traits and the phonetic content of speech \cite{chorowski_2019_unsupervised}. Due to this property, in the recently organized ZeroSpeech 2019 ``TTS Without T'' challenge \cite{Dunbar2019} VQ-VAE was a popular technique employed in many submissions \cite{Tjandra2019,Eloff2019} and led to best performance on the unit-discovery task.

However, training deep discrete latent representation models
is challenging due to the lack of a derivative function. %
One solution is to assume that the system is stochastic and optimize the mean output of discrete random units \cite{jang_categorical_2016}. In contrast, VQ-VAE is a deterministic discrete variable model which operates by solving an approximate on-line $k$-means clustering in which the gradient is approximated using the straight-through estimator \cite{bengio_estimating_2003}.
Therefore VQ-VAE can be trained using gradient descent like any other deep learning model. However we observe that it also
poses typical $k$-means challenges such as
sensitivity to initialization
and non-stationarity of clustered neural activations during training.
Moreover, the difficulty of $k$-means increases with the number of centroids, and the ability to encode the signal with a broad number of discrete codes is desirable for unsupervised representation learning~\cite{chorowski2019segmental}.

In this paper we discuss the training dynamics of VQ-VAE models, compare different learning methods, and propose alternatives.
First, we demonstrate the importance of proper initialization scale of the codebook relative to encoder outputs.
Then, we apply a deliberate, data-dependent initialization to the codebook. To address non-stationarity, we treat the problem as a sequence of static clustering problems, and periodically re-initialize the codebook from scratch with off-line clustering of recent encoder outputs. %
Lastly,
we \txt{evaluate} other \txt{possible training} improvements: separately tuned \txt{codebook} learning rates, and active \txt{balancing of latent feature magnitudes with codebook entries through batch normalization}~\cite{ioffe2015batch}. \txt{We establish that best results are obtained when all three mechanisms are present: batch normalization, data-dependent codebook re-initialization, and separately tuned learning rates.}

\txt{In the following sections we introduce the Vector Quantized bottleneck, describe the motivation for the proposed enhancements, and their relation to proposed alternations to the learning rate. We then present experiments on three different problem domains. Finally, we put our work in the broader context of representation learning.}

\section{Background}
Consider a multilayered neural network trained to solve either a supervised or an unsupervised task. Let one of its hidden layers be called a \emph{bottleneck} layer. We will refer to the values taken by the bottleneck layer as \emph{latent features}. The \emph{encoder} subnetwork takes input $x$ and computes latent features $e(x)$, which are then fed to the bottleneck. Its task is to remove spurious information from $e(x)$.

\txt{For instance, auto-encoding neural networks rely on a bottleneck layer placed between the encoder and decoder to prevent learning features which are a trivial transformation of their raw inputs. A common bottleneck choice is e.g. dimensionality reduction used in PCA and SVD. However, bottleneck layers are also useful in supervised learning, e.g. the Projected LSTM architecture for ASR introduces a dimensionality reduction layer between recurrent cells \cite{sak_long_2014}.}

Vector Quantized VAE~\cite{oord_neural_2017} introduced a bottleneck that quantizes neural activations $e(x)$, \txt{directly limiting the information content of the extracted features}. It learns a codebook of $K$ codewords $w_i\in\mathbb{R}^d$, and replaces each activation $e(x)$ with the nearest codeword:
\begin{equation}
      q(x) = \argmin_{w_i}\ \lVert e(x) - w_i\rVert.
\end{equation}

Regardless of the dimensionality of $q(x)$, the transmitted information is only about the index $i$ of the chosen codeword $w_i$ out of all $K$ codewords, and thus limited to $\log_2 K$ bits.  %

The network is trained using gradient optimization. Since quantization is non-differentiable, its derivative can only be approximated. VQ-VAE relies on the straight-through estimator $\partial\mathcal{L}/\partial e(x) \approx \partial\mathcal{L}/\partial q(x)$~\cite{bengio_estimating_2003}, and has additional loss terms for training the codebook~\cite{oord_neural_2017}:
\begin{equation}\label{eq:loss}
\begin{aligned}                                                         
    \mathcal{L} = L(q(x)) & + ||\mathrm{sg}[e(x)] - q(x)||^2 \\
                  & + \gamma||e(x) - \mathrm{sg}[q(x)]||^2,
\end{aligned}                                                           
\end{equation}
where $\mathrm{sg}[\cdot]$ denotes the stop gradient operation. The first term $L(q(x))$ corresponds to the task loss, for instance the negative log-likelihood of the reconstruction
\begin{equation}
L(q(x))=-\log p(x|q(x))
\end{equation}
in an autoencoder. The second term trains the codebook, by moving each codeword closer to the vector it replaced. As the result, it performs on-line $k$-means, with the codewords being the centroids. From now on we will use those terms interchangeably.
The third term optimizes the encoder to produce
representations close to the assigned codewords, with $\gamma\in\mathbb{R}$ being a scaling hyperparameter.

The quantized signal $q(x)$ is \txt{then processed by the rest of the network. In an auto-encoder, a \textit{decoder} subnetwork uses $q(x)$ to reconstruct the input $x$. However, the quantized signal can also be used to perform supervised classification.}

\begin{figure}[tb]
\begin{minipage}{0.0125\columnwidth}\hspace{0.01cm}\end{minipage}
\begin{minipage}{0.99\columnwidth}
\centering
{\resizebox{0.9\textwidth}{!}{\footnotesize\input{init.tex}}}\\(a)%
\end{minipage}\\
\begin{minipage}{0.49\columnwidth}
\vspace{3mm}
\centering
\includegraphics[scale=0.52]{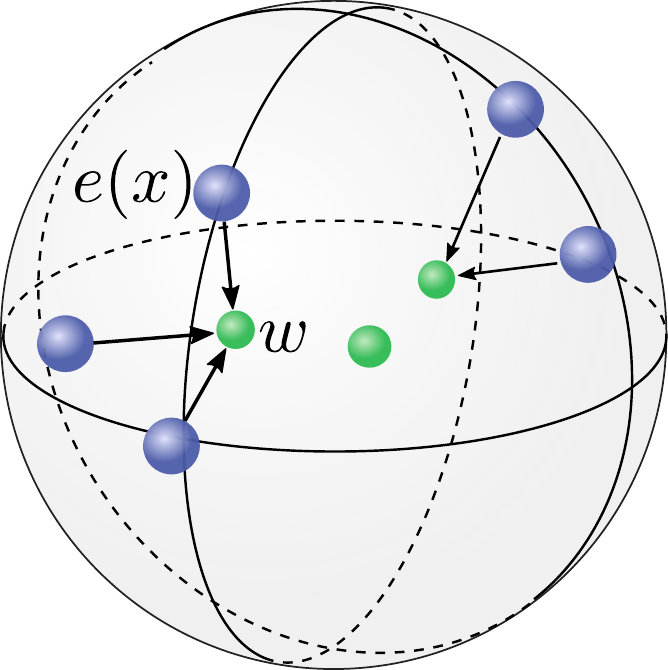}\\ (b) $\lvert e(x_i)\rvert \gg \lvert w_j\rvert$
\end{minipage}
\begin{minipage}{0.49\columnwidth}
\vspace{3mm}
\centering
\includegraphics[scale=0.52]{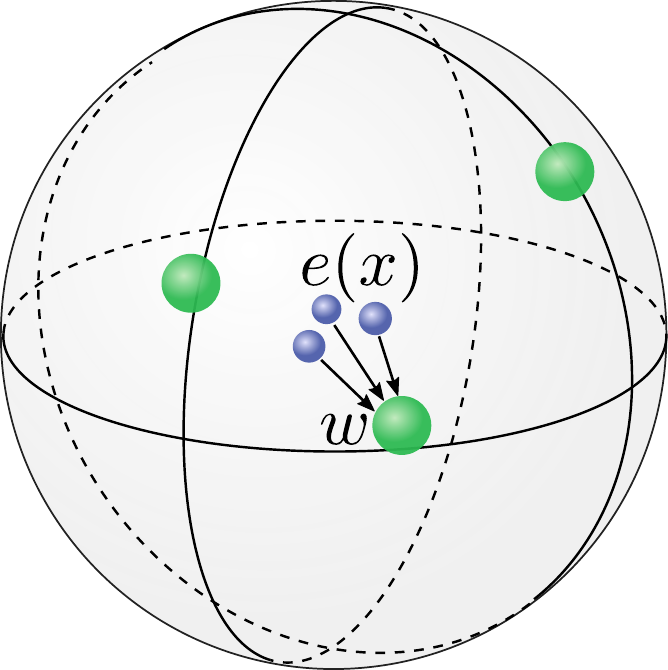}\\ (c) $\lvert e(x_i)\rvert \ll \lvert w_j\rvert$
\end{minipage}
\vspace{1mm}
\caption{The impact of scale of encoder outputs relative to the scale of codebook words shown in 3-D. \textbf{(a)} Relative scale of codewords $w$ and encoder outputs $e(x)$ impacts performance of mapping bottleneck features to symbols on a subset of ScribbleLens (see Section~\ref{sec:scribble} for details).
\textbf{(b)} If the encoder outputs are larger, multiple codewords are likely to be used. \textbf{(c)} If the encoder outputs are smaller, they tend to cluster and map to fewer codewords.}\label{fig:init-scale}
\end{figure}

\section{Intuitions about VQ training dynamics}
In this section we present practical observations on training of Vector Quantized bottlenecks
with the straight-through estimator, as proposed in~\cite{oord_neural_2017}.
First of all, only the second term of the loss~(\ref{eq:loss}) affects codewords,
implying that \txt{training algorithm only modifies the centroids that were selected.}
As a result, the training is prone to getting stuck in poor local optima,
in which only a \txt{small} subset of all codewords is in use.
This permanent low codebook usage prevents learning rich data representations,
because to\txt{o} little information is being passed through the bottleneck.

\subsection{Importance of proper scaling}
\label{sec:scaling}
We \txt{empirically} observe that in order \txt{to maximize codebook usage}, the codewords have to be %
smaller in norm than the encoder outputs $e(x)$ (Fig.~\ref{fig:init-scale}a).
During quantization, every output of the encoder $e(x)$ is matched to the nearest codeword~$w$. 
\txt{If the codewords are smaller in norm, the selection depends on the angular distance of the encoded representation and the codeword (Fig.~\ref{fig:init-scale}b), and many codewords are used.} However, if the codewords are larger in norm, all encoder outputs are likely to be assigned to the smallest codeword (Fig.~\ref{fig:init-scale}c). This blocks the training signal from reaching remaining codewords, and leads to underutilization of the codebook. %
\txt{Since the scale of feature activations can change during network training, we investigate the use of} batch normalization to \txt{enforce a magnitude} ratio between $e(x)$ and $w$.

\subsection{Batch data-dependent codebook updates}
Optimal $k$-means codebook initialization is data-dependent~\cite{arthur_kmeans_2007}, and  most usage scenarios of $k$-means assume stationarity of the data. However, during training of the encoder, the representations $e(x)$ changes with every iteration. These changes might be too rapid for the codebook to adapt, especially in the light of sparse gradient updates, that influence only the selected codewords.

For this reason, we propose to perform periodic data-dependent codebook \txt{reestimation during} %
a number of initial training steps (Algorithm~\ref{alg:reestim}).
\begin{algorithm}[tb]
\begin{algorithmic}
\For{$it \in {1, \ldots iterations}$}
    \State{$x = \text{encoder}(input)$}
    \State{$reservoir \gets \text{update\_reservoir}(reservoir, x)$}
    \If{$it < M_\textit{init}$}
        \State $quantized = x$
    \Else
        \If{$it\ \%\ r_\textit{reestim} == 0 \textbf{ and } it < M_\textit{init} + M_\textit{reestim}$}
            \State $codebook = \text{k-means++}(reservoir)$
        \EndIf
        \State $quantized = \text{nearest\_codes}(codebook)$
    \EndIf
    \LineComment{Continue forward pass}
\EndFor
\end{algorithmic}
\caption{Data-dependent Reestimated Vector Quantization}\label{alg:reestim}
\end{algorithm}
During training, we continually maintain a sample of outputs from the encoder using reservoir sampling \cite{Vitter1985reservoir}. During the first $M_{\textit{init}}$ warm-up iterations we perform no quantization, letting the network to initially stabilize outputs of the encoder.

After $M_{\textit{init}}$ iterations, \txt{we initialize the codebook by applying $k$-means++~ clustering~\cite{arthur_kmeans_2007} to samples collected in the reservoir. This ensures that all codewords are initially likely to be used.}
\txt{Afterward, we} periodically reestimate the codebook during the next $M_{\textit{reestim}}$ iterations, to help it keep up with changes to the \txt{distribution of encoder outputs.} %
By doing so, we treat the non-stationary problem of clustering \txt{rapidly changing} encoder outputs as a sequence of stationary problems. %

After this prolonged data-dependent initialization, we train the model using the regular Vector Quantization (VQ) algorithm. This fine-tunes the codebook \txt{and allows} the model to learn an efficient quantization which discards unimportant information. One of the advantages of data-dependent initialization is its robustness to the choice of the $M_{\textit{reestim}}$ parameter, leading \txt{to better} results than the vanilla training algorithm.

\subsection{EMA: an alternative training rule}

\txt{To improve the stability of VQ training, an} alternative codebook learning rule, based on exponential moving averages (EMA) was recently proposed \cite{oord_neural_2017,roy_theory_2018,razavi_generating_2019}. Below, we show that it is in fact equivalent to the the original VQ-VAE training criterion with a codebook-specific learning rate, and as such does not address the problems of codebook initialization and non-stationarity. 

Denote by $x_1,\ldots,x_{n_i}$ the training samples and by $e(x_1),\ldots,e(x_{n_i})$ the corresponding encoder outputs attracted by a codeword $w_i$ at training step $t$ 
(we omit $t$ for brevity).
It follows that a codeword $w_i$ is updated with those $n_i$ closest encoder outputs. Specifically, for every codeword $w_i$, the mean code $m_i$ and usage count $N_i$ is tracked:
\begin{align}
\begin{split}
N_i &\gets N_{i} \cdot \gamma + n_i (1 - \gamma), \\ %
m_i &\gets m_{i} \cdot \gamma + \sum_j^{n_i} e(x_j)(1 - \gamma),\label{eq:ema2}\\  %
w_i &\gets \frac{m_i}{N_i}, %
\end{split}
\end{align}     
where $\gamma$ is a discount factor.

We establish that the difference between the EMA training rule and the ordinary VQ-VAE codebook training algorithm is mostly adaptive rescaling of the codebook learning rate. 
\begin{proposition}
The EMA update rule (\ref{eq:ema2}) with constant usage counts $N_i=1$ is equivalent to an SGD update for ordinary loss (\ref{eq:loss}) with a rescaling learning rate $\alpha = (1-\gamma)/2$.
\end{proposition}
\begin{proof}
Let $x$ be an input example such that $e(x)$ maps to a codebook word $q(x)=w_i$. The SGD gradient update for a codeword $w_i$ and the ordinary VQ loss~(\ref{eq:loss}) is
\begin{equation}\begin{split}
\frac{\partial\mathcal{L}}{\partial w_i} =
\frac{\partial}{\partial w_i}\bigg[
   L(q(x)) & + ||\mathrm{sg}[e(x)] - q(x)||^2 \\
   & + \gamma||e(x) - \mathrm{sg}[q(x)]||^2\bigg].
\end{split}\end{equation}

Only the second term concerns the codebook, and a given codeword $w_i$ is only affected by the examples that were mapped to it.
Thus the gradient simplifies to
\begin{equation}
\frac{\partial\mathcal{L}}{\partial w_i} =
2\left( w_i - e(x_j)\right),
\end{equation}
which inserted into the gradient update
\begin{equation}
    w_i \gets w_i - \alpha \frac{\partial \mathcal{L}}{\partial w_i}
\end{equation}
with substitution $\gamma = (1 - 2\alpha)$ yields
\begin{equation}
w_i \gets w_i\cdot\gamma + e(x_j)(1 - \gamma).
\end{equation}
\end{proof}
Values reported for $\gamma$ in~\cite{oord_neural_2017} indicate that higher learning rates for the codebook might be beneficiary for ordinary VQ training. The $\gamma$ hyperparameter should be set carefully in accordance with learning rate. In the next section, we investigate the effect of increasing the learning rate $10\times$.

We now see that when EMA tracks codeword usage running means $N_i$ it effectively rescales the learning rate individually for every codeword according to its demand. If in a given batch there are fewer codewords $w_i$ used than the running average, the learning rate will be lower. Conversely, if the usage increased, the learning rate will be higher.
In the next section we compare EMA with ordinary VQ-VAE with rescaled codebook learning rates, and see that they lead to slightly different outcomes.

\section{Experiments}
We compare the proposed data-dependent codebook reestimation with vanilla codebook learning algorithm through minimization of loss (\ref{eq:loss}), and rescaled learning rate for the codebook, either explicitly or through EMA. In light of the scaling mismatch issue between codewords and encoded outputs from Section~\ref{sec:scaling}, we experiment with batch normalization~\cite{ioffe2015batch}.

In order to test the proposed method in a broader context, we perform experiments on images (CIFAR-10), speech (Wall Street Journal), and handwriting (ScribbLelens~\cite{Dolfing20}). We investigate static and temporal domains in supervised and unsupervised fashion for different modalities\footnote{
Source code for replication available at \url{https://github.com/distsup/DistSup}}.
\begin{table}[t]
\caption{Higher codebook usage (perplexity) leads to lower phoneme error rate (PER) on the supervised WSJ task}\label{tab:wsj}
\centering
\begin{tabular}{@{}p{4.6cm}rr@{}}
\toprule
Model                                      & Dev PER      & Code Perplexity \\
\midrule
vanilla VQ                                 & 16.9         & 16 \\
+ BN                                       & 17.6         & 16.5 \\
+ BN + codebook LR                         & 13.1         & 151 \\
+ BN + EMA                                 & 10.5         & 61.8 \\
+ BN + data-dep reest                      & 10.4         & 393 \\
\mbox{+ BN + data-dep reest + codebook LR} & \textbf{9.8} & \textbf{574} \\
\midrule
no bottleneck                              & 11.6         & \textit{N/A} \\
\bottomrule
\end{tabular}
\end{table}

\begin{figure}[t]
\centering
{\footnotesize \input{wsj_per_ema.tex}}
\caption{Wall Street Journal Dev93 supervised phoneme error rate (PER)}
\label{fig:wsj_training_curves}
\end{figure}
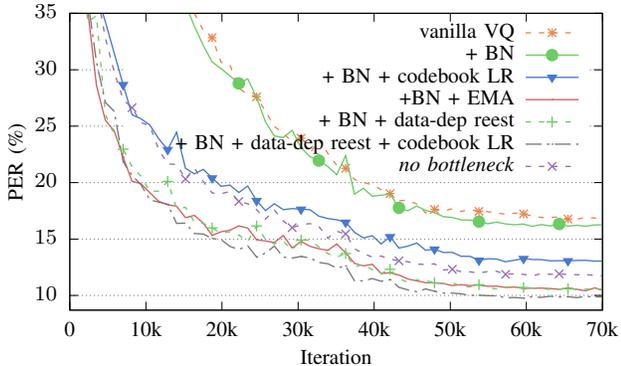

\subsection{Speech: Wall Street Journal}
We train a supervised speech recognition model and report phoneme error rates (PER).
Wall Street Journal (WSJ) is a corpus of news articles read by professional speakers, popular in the domain of speech recognition.
Our model is similar to Deep Speech 2~\cite{amodei2016deep} with an addition of a bottleneck.
The encoder is composed of a stack of strided convolutions followed by bidirectional recurrent layers, 
Encoded features are put through a VQ bottleneck with a single codebook. It quantizes the encoder's output at every timestep.
Finally, the quantized signal is passed through a $1\times 1$ convolution and \textit{SoftMax} followed by the CTC loss~\cite{graves_connectionist_2006}.

The inputs to the model are 80 Mel-filterbank bands with energy features extracted every 10ms with $\Delta$, $\Delta$-$\Delta$s, and global cepstral mean and variance normalization (CMVN). 
We train the model on 64-frame chunks with targets derived form forced alignments obtained using Tri3b model from Kaldi WSJ recipe~\cite{Povey_ASRU2011}. To get a clear picture of the VQ-VAE training process, we refrain from using regularization, which otherwise would be necessary for a small corpus like WSJ.

The model \textit{no bottleneck} serves as a baseline for quantized models. We contrast six variants of VQ training: base \emph{vanilla}, with batch normalization (\emph{BN}) inserted before the quantization, with BN and increased \emph{codebook learning rate}, using the \emph{EMA} training rule, and with \emph{data-dependent codebook reestimation} with and without \txt{increased codebook learning rates}.
We run hyperparameter searches over the learning rate and discount factor $\gamma$ for \emph{codebook learning rate} and \emph{EMA} models respectively. For EMA, we report results for the best models. For codebook learning rates, we establish a rule of thumb to set them to $10\times$ the regular learning rates. Wrong adjustment of these parameters could make the models worse than the \emph{vanilla} baseline.

We report phoneme error rates and codebook perplexity with no external language models in Table~\ref{tab:wsj} and Figure~\ref{fig:wsj_training_curves}. We notice that enabling the vanilla VQ bottleneck hurts the performance of the model.
Batch normalization \txt{together with} separate codebook learning rate help, but best results \txt{also} require data-dependent codebook reestimation. This model not only offers the highest codebook usage (\txt{indicated by highest} perplexity), but also reaches the error rate below the continuous baseline, demonstrating the regularizing effect of the information bottleneck.

\subsection{Handwriting: ScribbleLens}
\label{sec:scribble}
ScribbleLens~\cite{Scribble20,Dolfing20} is a freely available corpus of old handwritten Dutch from late 16th to mid-18th century, the time of Dutch East India company. It is segmented into lines of handwriting (Fig.~\ref{fig:tasman}). A subset of ScribbleLens is annotated with transcriptions.

\begin{figure}[tb]
    \centering
    \includegraphics[trim={0.5cm 0.35mm 6cm 0.3mm},clip,width=0.95\columnwidth]{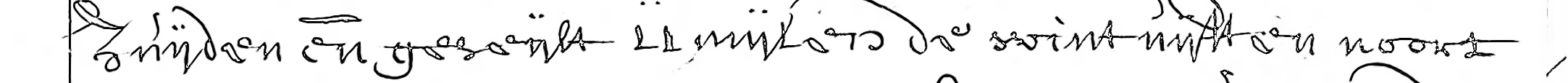}
    \caption{A sample training line from the ScribbleLens corpus~\cite{Scribble20,Dolfing20}}
    \label{fig:tasman}
\end{figure}

\txt{We treat lines of handwriting as multidimensional input features on which we train an unsupervised autoencoder. We use the same encoder architecture as for Wall Street Journal, followed by}
a VQ bottleneck with a single codebook (size $K=4096$), and a conditional PixelCNN~\cite{oord_pixel_2016} decoder. Input lines are scaled to be 32-pixels high. The encoder extracts a latent vector every four pixel columns.

\txt{Used alone, batch} normalization has a minor detrimental effect. Increased codebook learning rate and data-dependent codebook reestimation improve latent code usage and lead to lower reconstruction log-likelihoods measured in bits per dimension (BPD; Table~\ref{tab:scrib}). Again, improvements are related to a higher usage of the bottleneck capacity. 

We also report test set evidence lower bound (ELBO) expressed in bits per dimension~\cite{kingma_auto_2013,oord_neural_2017}:
\begin{equation}
\begin{split}
    \log p(x) =& \log \sum_{z}p(x|z)p(z) \geq ELBO = \\
    =& \log \mathbb{E}_{z~q(z|x)} p(x|z) - KL(q(z|x)||p(z)) = \\
    =& \log(p(x|z=q(x)) +\log p(z=q(x)).
\end{split}
\end{equation}
The last transition follows from the fact that $q(x)$ gives a deterministic encoding, $p(x|z)$ is the reconstruction loss,
and $p(z)$ is the prior latent code probability. We consider two formulations of the prior over latent vectors: a \emph{uniform} prior $p(z) = 1/K$, and a \emph{unigram} prior of observed codeword frequencies. The negative of ELBO (NELBO) has an intuitive interpretation: to transmit a data sample, we need to transmit the information contained in the latent features plus the information required to reconstruct the data sample from the conditional decoder probability. For ease of comparison with BPD, in Table~\ref{tab:scrib} we report the NELBOs in bits. The Uniform NELBO is computed as $\textrm{BPD} + \log_2(4096)/32/4$, i.e. the cost of reconstructing a single pixel augmented by the cost of transmitting a single latent amortized over all pixels this latent corresponds to, while for Unigram NELBO we account for the actual codebook us\txt{ag}e by computing $\textrm{BPD} + \log_2(\textrm{perplexity})/32/4$.

Ideally, an increase in bottleneck bandwidth should match the increase in reconstruction log-likelihood, keeping the ELBO constant. However, comparing the uniform and unigram NELBOs, we see that improvements in codebook usage translate to smaller likelihood improvements. This may be caused by imperfection of the PixelCNN decoder, unable to fully use the information in the latents.

\begin{table}[t]
    \centering
    \caption{On ScribbleLens, data-dependent codebook reestimation results in the fullest codebook usage, leading to the best reconstruction measured in bits per dim. Uniform and unigram ELBOs show that increased bottleneck bandwidth results in a smaller reconstruction likelihood improvement.}
    \label{tab:scrib}
    \addtolength{\tabcolsep}{-1pt}
    \begin{tabular}{@{}lrrrr@{}}
    \toprule
    Model & \makecell[r]{Rec.\\BPD} & Pplx & \makecell[c]{Uniform\\NELBO\\bits} & \makecell[c]{Unigram\\NELBO\\bits} \\
    \midrule
    vanilla VQ            &   0.213 &         322 &            0.307 &            \textbf{0.278} \\
    + BN                  &   0.216 &         260 &            0.309 &            \textbf{0.278} \\
    + BN + codebook LR    &   0.212 &         432 &            0.306 &            0.281 \\
    + BN + EMA            &   0.207 &        1118 &            0.301 &            0.287 \\ 
    + BN + data-dep reest &   \textbf{0.200} &        2388 &            \textbf{0.294} &            0.288 \\
    + BN + data-dep reest           &
    \multirow{2}{*}{\textbf{0.200}} &
    \multirow{2}{*}{\textbf{2446}} &
    \multirow{2}{*}{\textbf{0.294}} &
    \multirow{2}{*}{0.288} \\
    \hspace{0.65cm} + codebook LR  & \\

    \bottomrule
    \end{tabular}
    \addtolength{\tabcolsep}{-1pt} 
\end{table}

\subsection{Images: CIFAR-10}
\txt{On CIFAR-10 we train an unsupervised auto-encoder.} %
\txt{The model contains a convolutional encoder with residual connections, followed by  a VQ bottleneck and a deconvolutional decoder}. Our model follows closely~\cite{oord_neural_2017}, and we describe only the differences.
The network models a discrete distribution over $256$ output levels for RGB channels with \textit{SoftMax}. %
We use the Adam optimizer~\cite{kingma2014adam} with learning rate $5\mathrm{e}{-4}$. \txt{We also use Polyak checkpoint averaging in which the model is evaluated on parameters that are an exponential moving average of parameters during the most recent training steps \cite{polyak1992acceleration}}.

As even small images convey a lot of information, 
we consider two settings with multiple independent codebooks: 8 $\times$ 128 codewords (56 bits), and 5 $\times$ 1024 codewords (50 bits; Table~\ref{tab:cifar10}). 
Each encoded $e(x)$ is split into equally-sized blocks that are quantized independently.
This allows us to achieve high bottleneck bitrates with small codebooks, facilitating $k$-means training, which increases in difficulty with the number of centroids.
\begin{table}[t]
    \centering
    \caption{Unsupervised bits/dim (BPD) on CIFAR-10 test set for $8\times128$-codeword and $5\times1024$-codeword VQ bottlenecks (median of five runs)}
    \label{tab:cifar10}
    \begin{tabular}{@{}l  r r @{}}
    \toprule
    \multirow{2}{*}{\bf Model}  & \multicolumn{2}{c}{Reconstruction BPD}\\
                                & $8\times 128$ & $5\times 1024$\\
    \midrule
     VQ-VAE                     & 4.679       & 4.811       \\
    + BN                        & 4.677       & 4.809       \\
    + BN + codebook LR          & 4.677       & \textbf{4.796}      \\
    + BN + EMA                  & 4.683       & 4.798       \\
    + BN + data-dep reest       & 4.665       & 4.807       \\
    + BN + data-dep reest + codebook LR       & \textbf{4.659}        & 4.802      \\
    \midrule
    VQ-VAE \cite{oord_neural_2017} & \multicolumn{2}{r}{4.67 ($10\times 512$)} \\
    \bottomrule
    \end{tabular}
\end{table}

\begin{figure}[t]
\centering
{\footnotesize \input{cifar10.tex}}
\caption{Unsupervised bits/dim (BPD) on CIFAR-10 test set for $8\times128$-codeword model during training with checkpoint averaging.
         During initial data-dependent reestimation iterations, BPD is higher for checkpoint-averaged models, because each codebook re-initialization breaks the averages. However, when the reestimation period is over, these models converge faster than others.}
\label{fig:cifar10_convergence}
\end{figure}
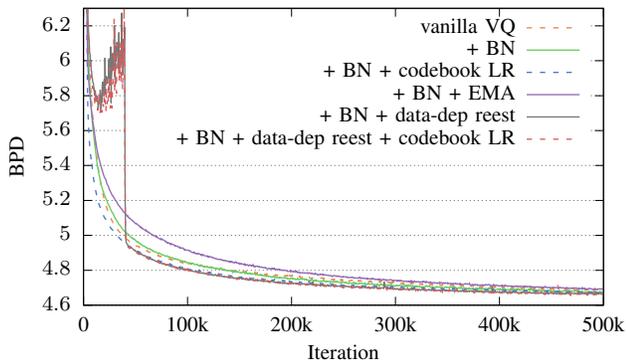

Table~\ref{tab:cifar10} summarizes test set bits/dim (BPD) for the models trained on the CIFAR-10 dataset. \txt{For this model,} placing a BN layer before the quantization leads to minor improvements in BPD in all settings. All models ultimately converge to similar values. Evaluating BPD during training (Figure~\ref{fig:cifar10_convergence}) shows that data-dependent reestimation improves convergence. The initial high values of BPD seen in the figure is caused by Polyak checkpoint averaging.

\section{Related Work}
Representation learning is a long-standing problem in machine learning. In auto-encoding systems, trained for input reconstruction a bottleneck  prevents the system from learning a trivial data representation, such as a multiplication by an orthonormal matrix. For instance, to prevent trivial encodings classical techniques, such as PCA or classical neural auto-encoders %
enforce a reduction of the dimensionality of the data. However, other constraints, such as sparsity \cite{olshausen_emergence_1996} or non-negativeness \cite{lee_learning_1999} \txt{have also been explored}.

The Variational Autoencdoer \cite{kingma_auto_2013} can be seen as an auto-encoder whose bottleneck limits the expected number of bits used to encode each sample. It therefore implements the information bottleneck~\cite{tishby1999information,alemi_deep_2016} principle. The limitation on encoding bitrate is made explicit in discrete latent variable models, such as the VQ-VAE~\cite{wu_variational_2018}. These have been shown to extract meaningful latent representations of speech in an unsupervised fashion~\cite{chorowski_2019_unsupervised}.

Deep Learning models are trained by gradient descent on a loss function. Discrete models by definition do not have a gradient. One possible solution is to treat them as stochastic, and optimize the expected value of the loss, which varies smoothly with model parameters and therefore has a well-defined gradient. This approach is employed in the Gumbel-Softmax reparametrization trick for the VAE \cite{jang_categorical_2016}, which yields a low-variance, but high bias estimator for the gradient. The gradient can also be estimated using the REINFORCE algorithm \cite{williams_simple_1992}, which is unbiased but has a high variance. The two approaches can be successfully combined \cite{tucker_rebar_2017}.

In contrast the VQ-VAE \cite{oord_neural_2017} operates in a deterministic mode, and computes a descent direction using backpropagation with the straight-through estimator \cite{bengio_estimating_2003}. Despite this approximation, it works well in practice.

Deterministic quantization with $k$-means can be seen as hard Expectation Maximization (EM)~\cite{roy_theory_2018}.
It can be replaced with soft EM, which performs probabilistic discrete quantization in order to improve training performance. Such bottleneck is an approximation to a continuous probabilistic bottleneck, similar to VAE~\cite{kingma_auto_2013}.
Probabilistic quantization can be applied only to improve training, and disabled during inference~\cite{sonderby_continuous_2017}. In addition,~\cite{sonderby_continuous_2017} introduces joint training of a prior on codewords, under which the most frequent codewords can be passed at a lower bit cost.

Classical acoustic unit discovery systems fit a probabilistic model which segments the speech and fits HMMs to each subword unit \cite{lee_nonparametric_2012}. Training of these systems requires Gibbs sampling or Variational Inference \cite{ondel_variational_2016}. Recently, they have been composed with neural acoustic models \cite{ebbers_hidden_2017,glarner_full_2018} by extending the VAE with a structured prior
\cite{johnson_composing_2016a}. In contrast, the VQ-VAE offers a much simpler approximation to the underlying probabilistic model, which can, however, be extended with an hmm-like segmental prior \cite{chorowski2019segmental}.

\section{Conclusions}
In this paper we have shown how VQ-VAE codebook learning can benefit from being viewed as on-line $k$-means clustering.
We propose the following \txt{enhancements} to the training procedure: increasing the codebook learning rate, \txt{enabling batch normalization prior to quantization}, and periodic batch codebook reestimations. Together, these \txt{three} changes stabilize training and lead to a broader codebook usage
and improved performance on supervised and unsupervised downstream tasks in image processing, speech recognition, and  handwritten unit discovery.
\section*{Acknowledgements}
We would like to thank Lucas Pagé-Caccia for help with CIFAR-10 experiments, and Benjamin van Niekerk for open sourcing his VQ-VAE experiments\footnote{\url{https://github.com/bshall/VectorQuantizedVAE}}.

The research reported here was conducted at the 2019 Frederick Jelinek Memorial Summer Workshop on Speech and Language Technologies, hosted at L'\'Ecole de Technologie Sup\'erieure (Montreal, Canada) and sponsored by Johns Hopkins University with unrestricted gifts from Amazon, Facebook, Google, and Microsoft. Authors also thank Polish National Science Center for funding under the
Sonata 2014/15/D/ST6/04402 grant the PLGrid project for computational resources on Prometheus cluster.

\bibliographystyle{IEEEbib}
\bibliography{refs}

\end{document}

%% file: init.tex
% GNUPLOT: LaTeX picture with Postscript
\begingroup
  \makeatletter
  \providecommand\color[2][]{%
    \GenericError{(gnuplot) \space\space\space\@spaces}{%
      Package color not loaded in conjunction with
      terminal option `colourtext'%
    }{See the gnuplot documentation for explanation.%
    }{Either use 'blacktext' in gnuplot or load the package
      color.sty in LaTeX.}%
    \renewcommand\color[2][]{}%
  }%
  \providecommand\includegraphics[2][]{%
    \GenericError{(gnuplot) \space\space\space\@spaces}{%
      Package graphicx or graphics not loaded%
    }{See the gnuplot documentation for explanation.%
    }{The gnuplot epslatex terminal needs graphicx.sty or graphics.sty.}%
    \renewcommand\includegraphics[2][]{}%
  }%
  \providecommand\rotatebox[2]{#2}%
  \@ifundefined{ifGPcolor}{%
    \newif\ifGPcolor
    \GPcolorfalse
  }{}%
  \@ifundefined{ifGPblacktext}{%
    \newif\ifGPblacktext
    \GPblacktexttrue
  }{}%
  % define a \g@addto@macro without @ in the name:
  \let\gplgaddtomacro\g@addto@macro
  % define empty templates for all commands taking text:
  \gdef\gplbacktext{}%
  \gdef\gplfronttext{}%
  \makeatother
  \ifGPblacktext
    % no textcolor at all
    \def\colorrgb#1{}%
    \def\colorgray#1{}%
  \else
    % gray or color?
    \ifGPcolor
      \def\colorrgb#1{\color[rgb]{#1}}%
      \def\colorgray#1{\color[gray]{#1}}%
      \expandafter\def\csname LTw\endcsname{\color{white}}%
      \expandafter\def\csname LTb\endcsname{\color{black}}%
      \expandafter\def\csname LTa\endcsname{\color{black}}%
      \expandafter\def\csname LT0\endcsname{\color[rgb]{1,0,0}}%
      \expandafter\def\csname LT1\endcsname{\color[rgb]{0,1,0}}%
      \expandafter\def\csname LT2\endcsname{\color[rgb]{0,0,1}}%
      \expandafter\def\csname LT3\endcsname{\color[rgb]{1,0,1}}%
      \expandafter\def\csname LT4\endcsname{\color[rgb]{0,1,1}}%
      \expandafter\def\csname LT5\endcsname{\color[rgb]{1,1,0}}%
      \expandafter\def\csname LT6\endcsname{\color[rgb]{0,0,0}}%
      \expandafter\def\csname LT7\endcsname{\color[rgb]{1,0.3,0}}%
      \expandafter\def\csname LT8\endcsname{\color[rgb]{0.5,0.5,0.5}}%
    \else
      % gray
      \def\colorrgb#1{\color{black}}%
      \def\colorgray#1{\color[gray]{#1}}%
      \expandafter\def\csname LTw\endcsname{\color{white}}%
      \expandafter\def\csname LTb\endcsname{\color{black}}%
      \expandafter\def\csname LTa\endcsname{\color{black}}%
      \expandafter\def\csname LT0\endcsname{\color{black}}%
      \expandafter\def\csname LT1\endcsname{\color{black}}%
      \expandafter\def\csname LT2\endcsname{\color{black}}%
      \expandafter\def\csname LT3\endcsname{\color{black}}%
      \expandafter\def\csname LT4\endcsname{\color{black}}%
      \expandafter\def\csname LT5\endcsname{\color{black}}%
      \expandafter\def\csname LT6\endcsname{\color{black}}%
      \expandafter\def\csname LT7\endcsname{\color{black}}%
      \expandafter\def\csname LT8\endcsname{\color{black}}%
    \fi
  \fi
    \setlength{\unitlength}{0.0500bp}%
    \ifx\gptboxheight\undefined%
      \newlength{\gptboxheight}%
      \newlength{\gptboxwidth}%
      \newsavebox{\gptboxtext}%
    \fi%
    \setlength{\fboxrule}{0.5pt}%
    \setlength{\fboxsep}{1pt}%
\begin{picture}(3600.00,2160.00)%
    \gplgaddtomacro\gplbacktext{%
      \csname LTb\endcsname%%
      \put(496,864){\makebox(0,0)[r]{\strut{}$0$}}%
      \put(496,1070){\makebox(0,0)[r]{\strut{}$10$}}%
      \put(496,1277){\makebox(0,0)[r]{\strut{}$20$}}%
      \put(496,1483){\makebox(0,0)[r]{\strut{}$30$}}%
      \put(496,1689){\makebox(0,0)[r]{\strut{}$40$}}%
      \put(496,1896){\makebox(0,0)[r]{\strut{}$50$}}%
      \put(592,704){\makebox(0,0){\strut{}$0.001$}}%
      \put(1008,704){\makebox(0,0){\strut{}$0.01$}}%
      \put(1424,704){\makebox(0,0){\strut{}$0.1$}}%
      \put(1839,704){\makebox(0,0){\strut{}$1$}}%
      \put(2255,704){\makebox(0,0){\strut{}$10$}}%
      \put(2671,704){\makebox(0,0){\strut{}$100$}}%
      \put(2767,864){\makebox(0,0)[l]{\strut{}$0$}}%
      \put(2767,1028){\makebox(0,0)[l]{\strut{}$20$}}%
      \put(2767,1193){\makebox(0,0)[l]{\strut{}$40$}}%
      \put(2767,1357){\makebox(0,0)[l]{\strut{}$60$}}%
      \put(2767,1522){\makebox(0,0)[l]{\strut{}$80$}}%
      \put(2767,1686){\makebox(0,0)[l]{\strut{}$100$}}%
      \put(2767,1851){\makebox(0,0)[l]{\strut{}$120$}}%
    }%
    \gplgaddtomacro\gplfronttext{%
      \colorrgb{0.28,0.47,0.82}%%
      \put(144,1431){\rotatebox{-270}{\makebox(0,0){\strut{}Char mapping acc (\%)}}}%
      \csname LTb\endcsname%%
      \put(3247,1431){\rotatebox{-270}{\makebox(0,0){\strut{}Used tokens}}}%
      \put(1631,528){\makebox(0,0){\strut{}Codebook scaling factor}}%
      \csname LTb\endcsname%%
      \put(1575,205){\makebox(0,0)[r]{\strut{}Char mapping acc (\%)}}%
      \csname LTb\endcsname%%
      \put(2934,205){\makebox(0,0)[r]{\strut{}Used tokens}}%
    }%
    \gplbacktext
    \put(0,0){\includegraphics{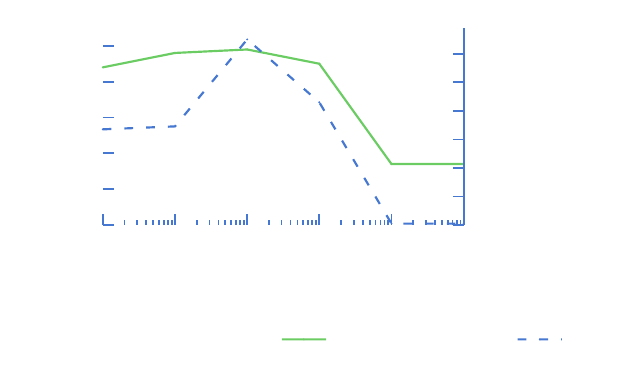}}%
    \put(0,0){\includegraphics{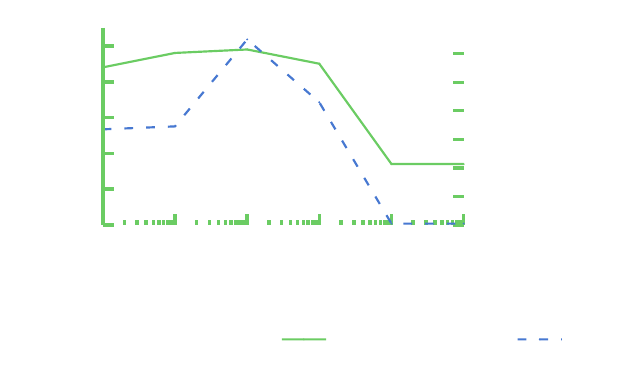}}%
    \put(0,0){\includegraphics{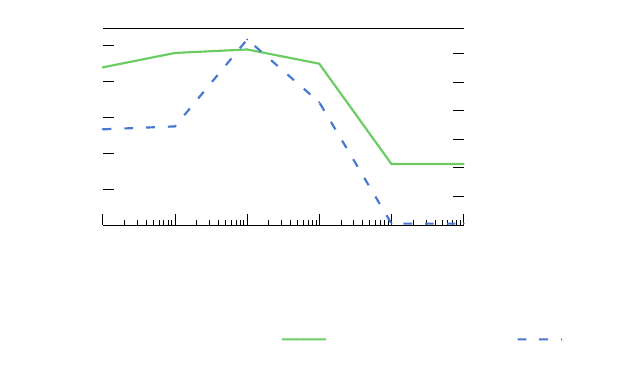}}%
    \gplfronttext
  \end{picture}%
\endgroup

%% file: wsj_per_ema.tex
% GNUPLOT: LaTeX picture with Postscript
\begingroup
  \makeatletter
  \providecommand\color[2][]{%
    \GenericError{(gnuplot) \space\space\space\@spaces}{%
      Package color not loaded in conjunction with
      terminal option `colourtext'%
    }{See the gnuplot documentation for explanation.%
    }{Either use 'blacktext' in gnuplot or load the package
      color.sty in LaTeX.}%
    \renewcommand\color[2][]{}%
  }%
  \providecommand\includegraphics[2][]{%
    \GenericError{(gnuplot) \space\space\space\@spaces}{%
      Package graphicx or graphics not loaded%
    }{See the gnuplot documentation for explanation.%
    }{The gnuplot epslatex terminal needs graphicx.sty or graphics.sty.}%
    \renewcommand\includegraphics[2][]{}%
  }%
  \providecommand\rotatebox[2]{#2}%
  \@ifundefined{ifGPcolor}{%
    \newif\ifGPcolor
    \GPcolorfalse
  }{}%
  \@ifundefined{ifGPblacktext}{%
    \newif\ifGPblacktext
    \GPblacktexttrue
  }{}%
  % define a \g@addto@macro without @ in the name:
  \let\gplgaddtomacro\g@addto@macro
  % define empty templates for all commands taking text:
  \gdef\gplbacktext{}%
  \gdef\gplfronttext{}%
  \makeatother
  \ifGPblacktext
    % no textcolor at all
    \def\colorrgb#1{}%
    \def\colorgray#1{}%
  \else
    % gray or color?
    \ifGPcolor
      \def\colorrgb#1{\color[rgb]{#1}}%
      \def\colorgray#1{\color[gray]{#1}}%
      \expandafter\def\csname LTw\endcsname{\color{white}}%
      \expandafter\def\csname LTb\endcsname{\color{black}}%
      \expandafter\def\csname LTa\endcsname{\color{black}}%
      \expandafter\def\csname LT0\endcsname{\color[rgb]{1,0,0}}%
      \expandafter\def\csname LT1\endcsname{\color[rgb]{0,1,0}}%
      \expandafter\def\csname LT2\endcsname{\color[rgb]{0,0,1}}%
      \expandafter\def\csname LT3\endcsname{\color[rgb]{1,0,1}}%
      \expandafter\def\csname LT4\endcsname{\color[rgb]{0,1,1}}%
      \expandafter\def\csname LT5\endcsname{\color[rgb]{1,1,0}}%
      \expandafter\def\csname LT6\endcsname{\color[rgb]{0,0,0}}%
      \expandafter\def\csname LT7\endcsname{\color[rgb]{1,0.3,0}}%
      \expandafter\def\csname LT8\endcsname{\color[rgb]{0.5,0.5,0.5}}%
    \else
      % gray
      \def\colorrgb#1{\color{black}}%
      \def\colorgray#1{\color[gray]{#1}}%
      \expandafter\def\csname LTw\endcsname{\color{white}}%
      \expandafter\def\csname LTb\endcsname{\color{black}}%
      \expandafter\def\csname LTa\endcsname{\color{black}}%
      \expandafter\def\csname LT0\endcsname{\color{black}}%
      \expandafter\def\csname LT1\endcsname{\color{black}}%
      \expandafter\def\csname LT2\endcsname{\color{black}}%
      \expandafter\def\csname LT3\endcsname{\color{black}}%
      \expandafter\def\csname LT4\endcsname{\color{black}}%
      \expandafter\def\csname LT5\endcsname{\color{black}}%
      \expandafter\def\csname LT6\endcsname{\color{black}}%
      \expandafter\def\csname LT7\endcsname{\color{black}}%
      \expandafter\def\csname LT8\endcsname{\color{black}}%
    \fi
  \fi
    \setlength{\unitlength}{0.0500bp}%
    \ifx\gptboxheight\undefined%
      \newlength{\gptboxheight}%
      \newlength{\gptboxwidth}%
      \newsavebox{\gptboxtext}%
    \fi%
    \setlength{\fboxrule}{0.5pt}%
    \setlength{\fboxsep}{1pt}%
\begin{picture}(4896.00,2880.00)%
    \gplgaddtomacro\gplbacktext{%
      \csname LTb\endcsname%%
      \put(496,591){\makebox(0,0)[r]{\strut{}$10$}}%
      \csname LTb\endcsname%%
      \put(496,1016){\makebox(0,0)[r]{\strut{}$15$}}%
      \csname LTb\endcsname%%
      \put(496,1442){\makebox(0,0)[r]{\strut{}$20$}}%
      \csname LTb\endcsname%%
      \put(496,1868){\makebox(0,0)[r]{\strut{}$25$}}%
      \csname LTb\endcsname%%
      \put(496,2293){\makebox(0,0)[r]{\strut{}$30$}}%
      \csname LTb\endcsname%%
      \put(496,2719){\makebox(0,0)[r]{\strut{}$35$}}%
      \put(592,320){\makebox(0,0){\strut{}0}}%
      \put(1166,320){\makebox(0,0){\strut{}10k}}%
      \put(1739,320){\makebox(0,0){\strut{}20k}}%
      \put(2313,320){\makebox(0,0){\strut{}30k}}%
      \put(2886,320){\makebox(0,0){\strut{}40k}}%
      \put(3460,320){\makebox(0,0){\strut{}50k}}%
      \put(4033,320){\makebox(0,0){\strut{}60k}}%
      \put(4607,320){\makebox(0,0){\strut{}70k}}%
    }%
    \gplgaddtomacro\gplfronttext{%
      \csname LTb\endcsname%%
      \put(201,1599){\rotatebox{-270}{\makebox(0,0){\strut{}PER (\%)}}}%
      \put(2599,112){\makebox(0,0){\strut{}Iteration}}%
      \csname LTb\endcsname%%
      \put(3944,2572){\makebox(0,0)[r]{\strut{}vanilla VQ}}%
      \csname LTb\endcsname%%
      \put(3944,2404){\makebox(0,0)[r]{\strut{}+ BN}}%
      \csname LTb\endcsname%%
      \put(3944,2236){\makebox(0,0)[r]{\strut{}+ BN + codebook LR}}%
      \csname LTb\endcsname%%
      \put(3944,2068){\makebox(0,0)[r]{\strut{}+BN + EMA}}%
      \csname LTb\endcsname%%
      \put(3944,1900){\makebox(0,0)[r]{\strut{}+ BN + data-dep reest}}%
      \csname LTb\endcsname%%
      \put(3944,1732){\makebox(0,0)[r]{\strut{}+ BN + data-dep reest + codebook LR}}%
      \csname LTb\endcsname%%
      \put(3944,1564){\makebox(0,0)[r]{\strut{}\emph{no bottleneck}}}%
    }%
    \gplbacktext
    \put(0,0){\includegraphics{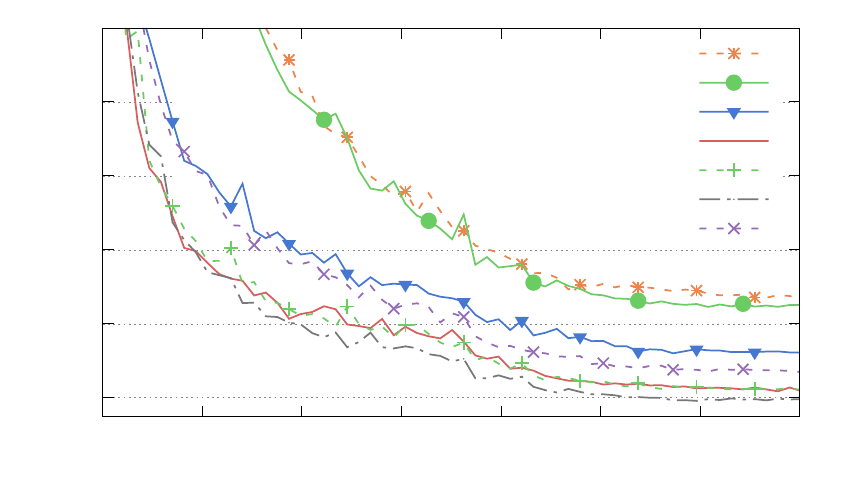}}%
    \gplfronttext
  \end{picture}%
\endgroup

%% file: cifar10.tex
% GNUPLOT: LaTeX picture with Postscript
\begingroup
  \makeatletter
  \providecommand\color[2][]{%
    \GenericError{(gnuplot) \space\space\space\@spaces}{%
      Package color not loaded in conjunction with
      terminal option `colourtext'%
    }{See the gnuplot documentation for explanation.%
    }{Either use 'blacktext' in gnuplot or load the package
      color.sty in LaTeX.}%
    \renewcommand\color[2][]{}%
  }%
  \providecommand\includegraphics[2][]{%
    \GenericError{(gnuplot) \space\space\space\@spaces}{%
      Package graphicx or graphics not loaded%
    }{See the gnuplot documentation for explanation.%
    }{The gnuplot epslatex terminal needs graphicx.sty or graphics.sty.}%
    \renewcommand\includegraphics[2][]{}%
  }%
  \providecommand\rotatebox[2]{#2}%
  \@ifundefined{ifGPcolor}{%
    \newif\ifGPcolor
    \GPcolorfalse
  }{}%
  \@ifundefined{ifGPblacktext}{%
    \newif\ifGPblacktext
    \GPblacktexttrue
  }{}%
  % define a \g@addto@macro without @ in the name:
  \let\gplgaddtomacro\g@addto@macro
  % define empty templates for all commands taking text:
  \gdef\gplbacktext{}%
  \gdef\gplfronttext{}%
  \makeatother
  \ifGPblacktext
    % no textcolor at all
    \def\colorrgb#1{}%
    \def\colorgray#1{}%
  \else
    % gray or color?
    \ifGPcolor
      \def\colorrgb#1{\color[rgb]{#1}}%
      \def\colorgray#1{\color[gray]{#1}}%
      \expandafter\def\csname LTw\endcsname{\color{white}}%
      \expandafter\def\csname LTb\endcsname{\color{black}}%
      \expandafter\def\csname LTa\endcsname{\color{black}}%
      \expandafter\def\csname LT0\endcsname{\color[rgb]{1,0,0}}%
      \expandafter\def\csname LT1\endcsname{\color[rgb]{0,1,0}}%
      \expandafter\def\csname LT2\endcsname{\color[rgb]{0,0,1}}%
      \expandafter\def\csname LT3\endcsname{\color[rgb]{1,0,1}}%
      \expandafter\def\csname LT4\endcsname{\color[rgb]{0,1,1}}%
      \expandafter\def\csname LT5\endcsname{\color[rgb]{1,1,0}}%
      \expandafter\def\csname LT6\endcsname{\color[rgb]{0,0,0}}%
      \expandafter\def\csname LT7\endcsname{\color[rgb]{1,0.3,0}}%
      \expandafter\def\csname LT8\endcsname{\color[rgb]{0.5,0.5,0.5}}%
    \else
      % gray
      \def\colorrgb#1{\color{black}}%
      \def\colorgray#1{\color[gray]{#1}}%
      \expandafter\def\csname LTw\endcsname{\color{white}}%
      \expandafter\def\csname LTb\endcsname{\color{black}}%
      \expandafter\def\csname LTa\endcsname{\color{black}}%
      \expandafter\def\csname LT0\endcsname{\color{black}}%
      \expandafter\def\csname LT1\endcsname{\color{black}}%
      \expandafter\def\csname LT2\endcsname{\color{black}}%
      \expandafter\def\csname LT3\endcsname{\color{black}}%
      \expandafter\def\csname LT4\endcsname{\color{black}}%
      \expandafter\def\csname LT5\endcsname{\color{black}}%
      \expandafter\def\csname LT6\endcsname{\color{black}}%
      \expandafter\def\csname LT7\endcsname{\color{black}}%
      \expandafter\def\csname LT8\endcsname{\color{black}}%
    \fi
  \fi
    \setlength{\unitlength}{0.0500bp}%
    \ifx\gptboxheight\undefined%
      \newlength{\gptboxheight}%
      \newlength{\gptboxwidth}%
      \newsavebox{\gptboxtext}%
    \fi%
    \setlength{\fboxrule}{0.5pt}%
    \setlength{\fboxsep}{1pt}%
\begin{picture}(4896.00,2880.00)%
    \gplgaddtomacro\gplbacktext{%
      \csname LTb\endcsname%%
      \put(592,480){\makebox(0,0)[r]{\strut{}$4.6$}}%
      \csname LTb\endcsname%%
      \put(592,743){\makebox(0,0)[r]{\strut{}$4.8$}}%
      \csname LTb\endcsname%%
      \put(592,1007){\makebox(0,0)[r]{\strut{}$5$}}%
      \csname LTb\endcsname%%
      \put(592,1270){\makebox(0,0)[r]{\strut{}$5.2$}}%
      \csname LTb\endcsname%%
      \put(592,1534){\makebox(0,0)[r]{\strut{}$5.4$}}%
      \csname LTb\endcsname%%
      \put(592,1797){\makebox(0,0)[r]{\strut{}$5.6$}}%
      \csname LTb\endcsname%%
      \put(592,2060){\makebox(0,0)[r]{\strut{}$5.8$}}%
      \csname LTb\endcsname%%
      \put(592,2324){\makebox(0,0)[r]{\strut{}$6$}}%
      \csname LTb\endcsname%%
      \put(592,2587){\makebox(0,0)[r]{\strut{}$6.2$}}%
      \put(688,320){\makebox(0,0){\strut{}0}}%
      \put(1472,320){\makebox(0,0){\strut{}100k}}%
      \put(2256,320){\makebox(0,0){\strut{}200k}}%
      \put(3039,320){\makebox(0,0){\strut{}300k}}%
      \put(3823,320){\makebox(0,0){\strut{}400k}}%
      \put(4607,320){\makebox(0,0){\strut{}500k}}%
    }%
    \gplgaddtomacro\gplfronttext{%
      \csname LTb\endcsname%%
      \put(201,1599){\rotatebox{-270}{\makebox(0,0){\strut{}BPD}}}%
      \put(2647,112){\makebox(0,0){\strut{}Iteration}}%
      \csname LTb\endcsname%%
      \put(3944,2572){\makebox(0,0)[r]{\strut{}vanilla VQ}}%
      \csname LTb\endcsname%%
      \put(3944,2404){\makebox(0,0)[r]{\strut{}+ BN}}%
      \csname LTb\endcsname%%
      \put(3944,2236){\makebox(0,0)[r]{\strut{}+ BN + codebook LR}}%
      \csname LTb\endcsname%%
      \put(3944,2068){\makebox(0,0)[r]{\strut{}+ BN + EMA}}%
      \csname LTb\endcsname%%
      \put(3944,1900){\makebox(0,0)[r]{\strut{}+ BN + data-dep reest}}%
      \csname LTb\endcsname%%
      \put(3944,1732){\makebox(0,0)[r]{\strut{}+ BN + data-dep reest + codebook LR}}%
    }%
    \gplbacktext
    \put(0,0){\includegraphics{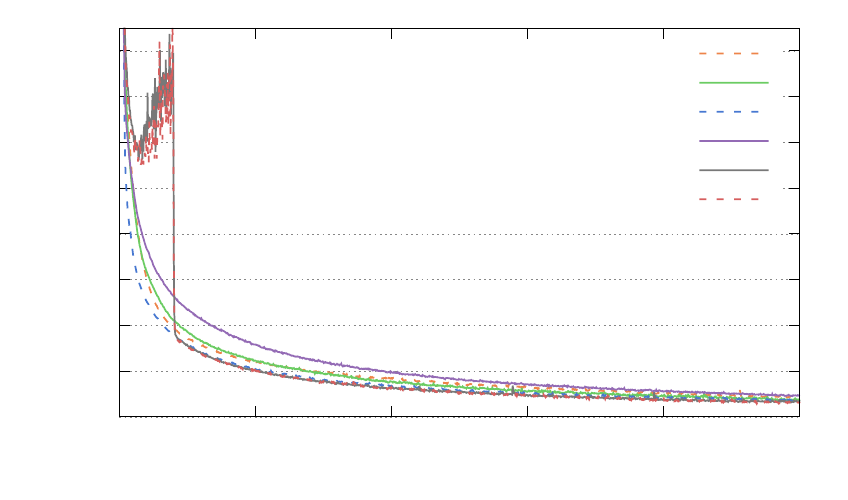}}%
    \gplfronttext
  \end{picture}%
\endgroup

%% file: robust.bbl
\begin{thebibliography}{10}

\bibitem{tishby1999information}
Naftali Tishby, Fernando~C Pereira, and William Bialek,
\newblock ``The information bottleneck method,''
\newblock in {\em The 37th annual Allerton Conference on Communication,
  Control, and Computing}, 1999, pp. 368--377.

\bibitem{vesely2012language}
Karel Vesel{\`y}, Martin Karafi{\'a}t, Franti{\v{s}}ek Gr{\'e}zl, Milo{\v{s}}
  Janda, and Ekaterina Egorova,
\newblock ``The language-independent bottleneck features,''
\newblock in {\em Proc. Spoken Language Technology Workshop (SLT)}, 2012, pp.
  336--341.

\bibitem{yu2011improved}
Dong Yu and Michael~L Seltzer,
\newblock ``Improved bottleneck features using pretrained deep neural
  networks,''
\newblock in {\em Interspeech}, 2011.

\bibitem{kingma_auto_2013}
Diederik~P. Kingma and Max Welling,
\newblock ``Auto-encoding variational bayes,''
\newblock {\em arXiv preprint arXiv:1312.6114}, 2013.

\bibitem{oord_neural_2017}
Aaron van~den Oord, Oriol Vinyals, and Koray Kavukcuoglu,
\newblock ``Neural {Discrete} {Representation} {Learning},''
\newblock {\em arXiv:1711.00937 [cs]}, Nov. 2017.

\bibitem{chorowski_2019_unsupervised}
J.~{Chorowski}, R.~J. {Weiss}, S.~{Bengio}, and A.~{van den Oord},
\newblock ``Unsupervised speech representation learning using wavenet
  autoencoders,''
\newblock {\em IEEE/ACM Transactions on Audio, Speech, and Language
  Processing}, vol. 27, no. 12, pp. 2041--2053, Dec 2019.

\bibitem{Dunbar2019}
Ewan Dunbar, Robin Algayres, Julien Karadayi, Mathieu Bernard, Juan Benjumea,
  Xuan-Nga Cao, Lucie Miskic, Charlotte Dugrain, Lucas Ondel, Alan~W. Black,
  Laurent Besacier, Sakriani Sakti, and Emmanuel Dupoux,
\newblock ``{The Zero Resource Speech Challenge 2019: TTS Without T},''
\newblock in {\em Proc. Interspeech 2019}, 2019, pp. 1088--1092.

\bibitem{Tjandra2019}
Andros Tjandra, Berrak Sisman, Mingyang Zhang, Sakriani Sakti, Haizhou Li, and
  Satoshi Nakamura,
\newblock ``{VQVAE Unsupervised Unit Discovery and Multi-Scale Code2Spec
  Inverter for Zerospeech Challenge 2019},''
\newblock in {\em Proc. Interspeech 2019}, 2019, pp. 1118--1122.

\bibitem{Eloff2019}
Ryan Eloff, André Nortje, Benjamin van Niekerk, Avashna Govender, Leanne
  Nortje, Arnu Pretorius, Elan van Biljon, Ewald van~der Westhuizen, Lisa van
  Staden, and Herman Kamper,
\newblock ``{Unsupervised Acoustic Unit Discovery for Speech Synthesis Using
  Discrete Latent-Variable Neural Networks},''
\newblock in {\em Proc. Interspeech 2019}, 2019, pp. 1103--1107.

\bibitem{jang_categorical_2016}
Eric Jang, Shixiang Gu, and Ben Poole,
\newblock ``Categorical {Reparameterization} with {Gumbel}-{Softmax},''
\newblock Nov. 2016.

\bibitem{bengio_estimating_2003}
Yoshua Bengio,
\newblock ``Estimating or {Propagating} {Gradients} {Through} {Stochastic}
  {Neurons},''
\newblock {\em arXiv:1305.2982 [cs]}, May 2013.

\bibitem{chorowski2019segmental}
Jan Chorowski, Nanxin Chen, Ricard Marxer, Hans Dolfing, Adrian Łańcucki,
  Guillaume Sanchez, Tanel Alumäe, and Antoine Laurent,
\newblock ``Unsupervised neural segmentation and clustering for unit discovery
  in sequential data,''
\newblock {\em Perception as Generative Reasoning Workshop, NeurIPS 2019},
  2019.

\bibitem{ioffe2015batch}
Sergey Ioffe and Christian Szegedy,
\newblock ``Batch {Normalization}: {Accelerating} {Deep} {Network} {Training}
  by {Reducing} {Internal} {Covariate} {Shift},''
\newblock {\em arXiv:1502.03167 [cs]}, Feb. 2015.

\bibitem{sak_long_2014}
Hasim Sak, Andrew~W. Senior, and Fran{\c{c}}oise Beaufays,
\newblock ``Long short-term memory recurrent neural network architectures for
  large scale acoustic modeling,''
\newblock in {\em {INTERSPEECH} 2014, 15th Annual Conference of the
  International Speech Communication Association, Singapore, September 14-18,
  2014}, Haizhou Li, Helen~M. Meng, Bin Ma, Engsiong Chng, and Lei Xie, Eds.
  2014, pp. 338--342, {ISCA}.

\bibitem{arthur_kmeans_2007}
David Arthur and Sergei Vassilvitskii,
\newblock ``K-means++: The advantages of careful seeding,''
\newblock in {\em Proceedings of the Eighteenth Annual ACM-SIAM Symposium on
  Discrete Algorithms}, Philadelphia, PA, USA, 2007, SODA '07, pp. 1027--1035,
  Society for Industrial and Applied Mathematics.

\bibitem{Vitter1985reservoir}
Jeffrey~S. Vitter,
\newblock ``Random sampling with a reservoir,''
\newblock {\em ACM Trans. Math. Softw.}, vol. 11, no. 1, pp. 37--57, Mar. 1985.

\bibitem{roy_theory_2018}
Aurko Roy, Ashish Vaswani, Arvind Neelakantan, and Niki Parmar,
\newblock ``Theory and {Experiments} on {Vector} {Quantized} {Autoencoders},''
\newblock {\em arXiv:1805.11063 [cs, stat]}, May 2018.

\bibitem{razavi_generating_2019}
Ali Razavi, Aaron van~den Oord, and Oriol Vinyals,
\newblock ``Generating {Diverse} {High}-{Fidelity} {Images} with
  {VQ}-{VAE}-2,''
\newblock {\em arXiv:1906.00446 [cs, stat]}, June 2019.

\bibitem{Dolfing20}
Hans~J.G.A. Dolfing, Jerome Bellegarda, Jan Chorowski, Ricard Marxer, and
  Antoine Laurent,
\newblock ``{The ``ScribbleLens'' Dutch historical handwriting corpus},''
\newblock in {\em Under review for: International Conference on Frontiers of
  Handwriting Recognition (ICFHR)}, 2020,
\newblock \url{http://www.openslr.org/84/}.

\bibitem{amodei2016deep}
Dario Amodei, Sundaram Ananthanarayanan, and Rishita~Anubhai et~al.,
\newblock ``Deep speech 2 : End-to-end speech recognition in english and
  mandarin,''
\newblock in {\em Proceedings of The 33rd International Conference on Machine
  Learning}, Maria~Florina Balcan and Kilian~Q. Weinberger, Eds., New York, New
  York, USA, 20--22 Jun 2016, vol.~48 of {\em Proceedings of Machine Learning
  Research}, pp. 173--182, PMLR.

\bibitem{graves_connectionist_2006}
Alex Graves, Santiago Fern\'{a}ndez, Faustino Gomez, and J\"{u}rgen
  Schmidhuber,
\newblock ``Connectionist temporal classification: Labelling unsegmented
  sequence data with recurrent neural networks,''
\newblock in {\em Proceedings of the 23rd International Conference on Machine
  Learning}, New York, NY, USA, 2006, ICML '06, pp. 369--376, ACM.

\bibitem{Povey_ASRU2011}
Daniel Povey, Arnab Ghoshal, Gilles Boulianne, Lukas Burget, Ondrej Glembek,
  Nagendra Goel, Mirko Hannemann, Petr Motlicek, Yanmin Qian, Petr Schwarz, Jan
  Silovsky, Georg Stemmer, and Karel Vesely,
\newblock ``The kaldi speech recognition toolkit,''
\newblock in {\em IEEE 2011 Workshop on Automatic Speech Recognition and
  Understanding}. Dec. 2011, IEEE Signal Processing Society,
\newblock IEEE Catalog No.: CFP11SRW-USB.

\bibitem{Scribble20}
Open Speech and Language Resources,
\newblock {\em ScribbleLens},
\newblock \url{http://www.openslr.org/84/}.

\bibitem{oord_pixel_2016}
A\"{a}ron Van Den~Oord, Nal Kalchbrenner, and Koray Kavukcuoglu,
\newblock ``Pixel recurrent neural networks,''
\newblock in {\em Proceedings of the 33rd International Conference on
  International Conference on Machine Learning - Volume 48}. 2016, ICML'16, pp.
  1747--1756, JMLR.org.

\bibitem{kingma2014adam}
Diederik Kingma and Jimmy Ba,
\newblock ``Adam: {A} {Method} for {Stochastic} {Optimization},''
\newblock {\em arXiv:1412.6980 [cs]}, Dec. 2014.

\bibitem{polyak1992acceleration}
Boris~T Polyak and Anatoli~B Juditsky,
\newblock ``Acceleration of stochastic approximation by averaging,''
\newblock {\em SIAM journal on control and optimization}, vol. 30, no. 4, pp.
  838--855, 1992.

\bibitem{olshausen_emergence_1996}
Bruno~A. Olshausen and David~J. Field,
\newblock ``Emergence of simple-cell receptive field properties by learning a
  sparse code for natural images,''
\newblock {\em Nature}, vol. 381, no. 6583, pp. 607--609, June 1996.

\bibitem{lee_learning_1999}
Daniel~D. Lee and H.~Sebastian Seung,
\newblock ``Learning the parts of objects by non-negative matrix
  factorization,''
\newblock {\em Nature}, vol. 401, no. 6755, pp. 788--791, Oct. 1999.

\bibitem{alemi_deep_2016}
Alexander~A. Alemi, Ian Fischer, Joshua~V. Dillon, and Kevin Murphy,
\newblock ``Deep {Variational} {Information} {Bottleneck},''
\newblock Dec. 2016.

\bibitem{wu_variational_2018}
Hanwei Wu and Markus Flierl,
\newblock ``Variational {Information} {Bottleneck} on {Vector} {Quantized}
  {Autoencoders},''
\newblock {\em arXiv:1808.01048 [cs, stat]}, Aug. 2018.

\bibitem{williams_simple_1992}
Ronald~J. Williams,
\newblock ``Simple statistical gradient-following algorithms for connectionist
  reinforcement learning,''
\newblock {\em Machine Learning}, vol. 8, no. 3-4, pp. 229--256, May 1992.

\bibitem{tucker_rebar_2017}
George Tucker, Andriy Mnih, Chris~J Maddison, John Lawson, and Jascha
  Sohl-Dickstein,
\newblock ``{REBAR}: {Low}-variance, unbiased gradient estimates for discrete
  latent variable models,''
\newblock in {\em Advances in {Neural} {Information} {Processing} {Systems}
  30}, I.~Guyon, U.~V. Luxburg, S.~Bengio, H.~Wallach, R.~Fergus,
  S.~Vishwanathan, and R.~Garnett, Eds., pp. 2627--2636. Curran Associates,
  Inc., 2017.

\bibitem{sonderby_continuous_2017}
Casper~Kaae S{\o}nderby, Ben Poole, and Andriy Mnih,
\newblock ``Continuous relaxation training of discrete latent variable image
  models,''
\newblock {\em Beysian DeepLearning workshop}, NIPS 201.

\bibitem{lee_nonparametric_2012}
Chia-ying Lee and James Glass,
\newblock ``A {Nonparametric} {Bayesian} {Approach} to {Acoustic} {Model}
  {Discovery},''
\newblock in {\em Proceedings of the 50th {Annual} {Meeting} of the
  {Association} for {Computational} {Linguistics} ({Volume} 1: {Long}
  {Papers})}, Jeju Island, Korea, July 2012, pp. 40--49, Association for
  Computational Linguistics.

\bibitem{ondel_variational_2016}
Lucas Ondel, Lukaš Burget, and Jan Černocký,
\newblock ``Variational {Inference} for {Acoustic} {Unit} {Discovery},''
\newblock {\em Procedia Computer Science}, vol. 81, pp. 80--86, 2016.

\bibitem{ebbers_hidden_2017}
Janek Ebbers, Jahn Heymann, Lukas Drude, Thomas Glarner, Reinhold Haeb-Umbach,
  and Bhiksha Raj,
\newblock ``Hidden {Markov} {Model} {Variational} {Autoencoder} for {Acoustic}
  {Unit} {Discovery},''
\newblock in {\em Interspeech 2017}. Aug. 2017, pp. 488--492, ISCA.

\bibitem{glarner_full_2018}
Thomas Glarner, Patrick Hanebrink, Janek Ebbers, and Reinhold Haeb-Umbach,
\newblock ``Full {Bayesian} {Hidden} {Markov} {Model} {Variational}
  {Autoencoder} for {Acoustic} {Unit} {Discovery},''
\newblock in {\em Interspeech 2018}. Sept. 2018, pp. 2688--2692, ISCA.

\bibitem{johnson_composing_2016a}
Matthew~J. Johnson, David Duvenaud, Alexander~B. Wiltschko, Sandeep~R. Datta,
  and Ryan~P. Adams,
\newblock ``Composing graphical models with neural networks for structured
  representations and fast inference,''
\newblock {\em arXiv:1603.06277 [stat]}, Mar. 2016.

\end{thebibliography}
